\documentclass[final,12pt]{alt2024} % Include author names

\usepackage{amsmath,amsfonts,stmaryrd,amssymb,enumitem}
\usepackage{algorithm,booktabs}

\usepackage{algorithmicx}
\usepackage{algpseudocode}
\usepackage{url}
\usepackage{hyperref}
\algnewcommand\algorithmicinput{\textbf{Input:}}
\algnewcommand\INPUT{\item[\algorithmicinput]}

\algnewcommand\algorithmicoutput{\textbf{Output:}}
\algnewcommand\OUTPUT{\item[\algorithmicoutput]}

\usepackage{pifont}
\usepackage{bbm}

\usepackage{tikz}
\usetikzlibrary{arrows}
\usetikzlibrary{calc}
\usetikzlibrary{shapes}
\usetikzlibrary{fit}
\usetikzlibrary{matrix} 
\usetikzlibrary{positioning}
\usetikzlibrary{decorations.pathreplacing}

\renewcommand{\eqref}[1]{Eq.~(\ref{eq:#1})}

\newcommand{\tabref}[1]{Table~\ref{tab:#1}}        
\newcommand{\secref}[1]{Section~\ref{sec:#1}}
\newcommand{\thmref}[1]{Theorem~\ref{thm:#1}}
\newcommand{\lemref}[1]{Lemma~\ref{lem:#1}}

\newcommand{\appref}[1]{Appendix \ref{app:#1}}

\renewcommand{\P}{\mathbb{P}}
\newcommand{\E}{\mathbb{E}}

\newcommand{\VS}{\mathrm{VS}}
\newcommand{\dom}{\mathrm{dom}}

\newcommand{\reals}{\mathbb{R}}
\newcommand{\nats}{\mathbb{N}}

\newcommand{\one}{\mathbbm{1}}

\newcommand{\cmark}{\ding{51}}%
\newcommand{\xmark}{\ding{55}}%

\renewcommand{\vec}[1]{\mathbf{#1}}

\newcommand{\cH}{\mathcal{H}}

\renewcommand{\pmod}{p'}

\title[Adaptive Combinatorial Maximization: 
  Beyond Approximate Greedy Policies]{Adaptive Combinatorial Maximization:\\
  Beyond Approximate Greedy Policies}
\usepackage{times}

\altauthor{%
 \Name{Shlomi Weitzman} \Email{shlomiweitzman@gmail.com}\\
 \addr Department of Computer Science, Ben-Gurion University of the Negev
 \AND
 \Name{Sivan Sabato} \Email{sabatos@mcmaster.ca}\\
 \addr Department of Computing and Software, McMaster University\\
 \addr Canada CIFAR AI Chair, Vector Institute\\
 \addr Department of Computer Science, Ben-Gurion University of the Negev
}

\begin{document}

\maketitle
\begin{abstract}
We study adaptive combinatorial maximization, which is a core challenge in machine learning, with applications in active learning as well as many other domains. We study the Bayesian setting, and consider the objectives of maximization under a cardinality constraint and minimum cost coverage. 
We provide new comprehensive approximation guarantees that subsume previous results, as well as considerably strengthen them. Our approximation guarantees simultaneously support the maximal gain ratio as well as near-submodular utility functions, and include both maximization under a cardinality constraint and a minimum cost coverage guarantee. In addition, we provided an approximation guarantee for a modified prior, which is crucial for obtaining active learning guarantees that do not depend on the smallest probability in the prior.
Moreover, we discover a new parameter of adaptive selection policies, which we term the \emph{maximal gain ratio}. We show that this parameter is strictly less restrictive than the greedy approximation parameter that has been used in previous approximation guarantees, and show that it can be used to provide stronger approximation guarantees than previous results. In particular, we show that the maximal gain ratio is never larger than the greedy approximation factor of a policy, and that it can be considerably smaller. This provides a new insight into the properties that make a policy useful for adaptive combinatorial maximization.
\end{abstract}

\begin{keywords}
  Combinatorial maximization, adaptive submodularity, approximate greedy policies
\end{keywords}

\section{Introduction}
\label{sec:intro}
Adaptive combinatorial maximization is a core challenge in machine learning, spanning a wide range of applications, including the fundamental challenges of active learning \citep{golovin2011adaptive} and adaptive experiment design \citep{doppa2021adaptive}, as well as many other applications, such as influence maximization in social networks \citep{seeman2013adaptive,tong2016adaptive}, movie recommendations \citep{mitrovic2019adaptive} and spectroscopy \citep{hino2020active}.

In this algorithmic problem, there are elements with hidden states. Elements are selected sequentially. Whenever an element is selected, its state is observed. Past observations can be used to make the next selection decision. The obtained utility from a specific selection process depends on the selected elements and their true states. For instance, consider a problem of selecting locations for placing radio towers \citep{asadpour2008stochastic,golovin2011adaptive}, where the goal is to ensure sufficient signal strength in the designated area. The states in this case are parameters of the transmission that can only be measured after a radio tower is placed. The next radio tower location can be selected based on the previous measurements.

A particularly important application of adaptive combinatorial application in machine learning is active learning \citep{mccallum1998employing}, in which the elements are examples, and the state of an element is its true label. The active learning algorithm selects which element to have labeled, so as to reveal its true label, and the objective is to identify the true classifier for the entire set of examples using a small number of labels.

We study the Bayesian setting, in which there is a known prior over the possible states of the elements, and the performance of the algorithm is measured with respect to the expectation over this prior. The goal of the algorithm is either to obtain a high expected utility using a limited number of element selections, also termed ``maximization under a cardinality constraint'' or alternatively, to obtain the maximal possible utility value using a small expected number of element selections. The latter is typically referred to as ``minimum cost coverage''. This objective is of particular importance for active learning, since in this application, the goal is to reach a specific target of finding the correct classifier, which can be mapped to a maximal value of the utility function. \cite{golovin2011adaptive} were the first to introduce the adaptive Bayesian formulation.

Calculating an optimal selection policy for either of the objectives mentioned above is computationally hard in the general case \citep[see, e.g.][]{NemhauserWoFi78,Wolsey82}. Thus, we seek instead to provide approximation guarantees with respect to an optimal policy. \cite{golovin2011adaptive} defined a class of utility functions called \emph{adaptive submodular}, which generalizes the non-adaptive notion of submodular set functions \citep{Edmonds70}. They showed that if the functions are also \emph{adaptive monotone}, then an \emph{approximately greedy} policy obtains a bounded approximation factor for the objective of maximization under a cardinality constraint. A greedy policy selects, in each round, the element that will obtain the largest increase in utility in expectation.  An approximate greedy policy may select an element that is maximal up to a given factor. \cite{golovinKrause17arxiv} provided an approximation guarantee for the minimum cost coverage objective. A smaller approximation factor was proved by \cite{esfandiari2021adaptivity}.

The guarantees mentioned above require that the utility function is adaptive submodular. \cite{fujii2019beyond} studied a relaxed version, using the \emph{adaptive submodularity ratio}, a property that quantifies the closeness of a utility function to being submodular. They provided an approximation guarantee for maximization under a cardinality constraint for greedy policies, that is parameterized by the adaptive submodularity ratio of the utility function. 

In this work, we provide a new perspective on the problem of adaptive combinatorial maximization, by providing comprehensive approximation guarantees that not only subsume previous results, but also considerably strengthen them.
We discover a new parameter of adaptive selection policies, which we term the \emph{maximal gain ratio}. We show that this parameter is strictly less restrictive than the greedy approximation parameter that has been used in previous approximation guarantees, and show that it can be used to provide stronger approximation guarantees than previous results. In particular, we show that the maximal gain ratio is never larger than the greedy approximation factor of a policy, and that it can be considerably smaller. 

The maximal gain ratio is the maximal ratio between the expected marginal gain of remaining elements at termination and the expected marginal gain of selected elements. In this way, it can take into account also selections that would seem out of order for a greedy policy, but that accomplish a similar goal. Moreover, we show that even for greedy policies, the value of the maximal gain ratio can be arbitrarily small. Thus, this provides a more refined understanding of what makes an adaptive selection policy successful. Our contributions are the following:
\begin{itemize}
\item Defining the maximal gain ratio, a new policy parameter;
\item Showing that this parameter is strictly less restrictive than the greedy approximation parameter;
\item An approximation guarantee for utility maximization under a cardinality constraint that is parameterized by the maximal gain ratio of a policy as well as the adaptive submodularity ratio of the utility function;
\item An approximation guarantee for the minimum cost coverage objective parameterized by the same properties;
\item A version of the minimum cost coverage guarantee that allows using the problem parameters with a modified prior first suggested by \cite{golovin2011adaptive}, leading to an improved approximation ratio for Bayesian active learning under priors with small probabilities. 
\end{itemize}
We summarize the main difference between our work and previous works mentioned above in \tabref{Comparison}. Our work is the first to support, at the same time, policies that are not greedy, utility functions that are not adaptive submodular, and the minimum cost coverage objective, as well as guarantees for the modified prior. 

\begin{table}[h]
\footnotesize
    \centering
      \begin{tabular}{cccccc}
        \toprule
         & Approximate & Nearly adaptive & Min cost & Support for & Beyond\\
         & greedy policies & submodular utility & coverage & modified prior & approximate greedy\\
        \midrule
        GK11 & \cmark\ & \xmark\ & (\cmark) & (\cmark) & \xmark\ \\
        FS19 & \xmark\ & \cmark\ & \xmark\ & \xmark\ & \xmark\ \\
        EKM21 & \xmark\ & \xmark\ & \cmark\ & \xmark\ & \xmark\ \\
        \midrule
        \textbf{This Work} & \cmark\ & \cmark\ & \cmark\ & \cmark\ & \cmark\ \\
        \bottomrule
    \end{tabular}
      \caption{Comparison to previous works. GK11: \cite{golovin2011adaptive}, FS19: \cite{fujii2019beyond}, EKM21: \cite{esfandiari2021adaptivity}. (\cmark) indicates a looser guarantee. }
    \label{tab:Comparison}
\end{table}

\paragraph{Paper structure.} In \secref{Related Work}, we complement the discussion above on related work. Definitions and necessary background are provided in \secref{Preliminaries}. We provide a summary of our main results in \secref{main}. The maximal gain ratio is presented and analyzed in \secref{beta}. The proofs of the approximation guarantees for the true prior based on the maximal gain ratio are provided in \appref{New guarantees for nearly adaptive submodular functions}. The modified prior and its relevance to active learning are discussed in \secref{modified}, where the proof of our guarantee for this prior is also provided, with some parts deferred to \appref{modified}. We summarize in \secref{Conclusions}. 

\section{Related Work}
\label{sec:Related Work}
The concept of adaptive submodular functions was first introduced in \cite{golovin2011adaptive}. They showed that this class of functions allows deriving guarantees for several settings, including adaptive submodular maximization with respect to a Bayesian prior. The latter setting was further studied in \cite{fujii2019beyond} and \cite{esfandiari2021adaptivity}, who provided additional guarantees, as discussed above. \cite{fujii2019beyond}  introduced the adaptive submodularity ratio, which generalized the submodularity ratio proposed by \cite{das2011submodular} for non-adaptive set functions to the adaptive setting, and provided guarantees parameterized by this quantity. The notion of \emph{curvature}, first introduced in \cite{conforti1984submodular} for set functions in the non-adaptive combinatorial maximization setting, allows further improving the guarantees and algorithms for submodular maximization under a cardinality constraint \citep{vondrak10,balkanski2018approximation}. 
  
Adaptive submodular maximization has been studied in other settings as well, such as a worst-case non-Bayesian setting with pointwise submodularity \citep{guillory2010interactive} with applications to active learning \citep{CuongLeYe14}, outcome dependent costs \citep{Sabato18}, adaptive maximization of non-monotone submodular functions \citep{gotovos2015non,amanatidis2020fast, TANG21}, streaming adaptive maximization \citep{tang2023streaming}, and robust adaptive submodular maximization \cite{tang2022robust}.
It has been applied to many settings, including batch mode active learning \citep{chen2013near}, bandit algorithms \citep{gabillon2013adaptive}, and reinforcement learning \citep{wu2022adaptive}.

\section{Preliminaries}
\label{sec:Preliminaries}
In this section, we give necessary preliminaries. Definitions and notation are provided in \secref{Definitions and Notation}. Previously studied properties of utility functions and policies are provided in \secref{Properties of utility functions and policies}. Guarantees from previous works are given in \secref{Previous approximation guarantees}.

\subsection{Definitions and Notation}
\label{sec:Definitions and Notation}
Let $V$ be a finite ground set of elements.
Let $Y$ be a finite set of the possible states of each of the elements in $V$.
A \emph{realization} is a map $\phi:V \to Y$ that associates each ground element with a state.
A \emph{partial realization} $\psi: V \rightharpoonup Y$ is a map from a subset of $V$ to $Y$. Its domain is denoted by $\dom(\psi)$.
We denote the set of all realizations by $\Phi_V$ and the set of all partial realizations by $\Psi_V$.
For any realization $\phi \in \Phi_V$ and partial realization $\psi \in \Psi_V$, we write $\phi \sim \psi$ to indicate that for every $x \in \dom(\psi)$, $\psi(x) = \phi(x)$.
For any two partial realizations $\psi$ and $\psi'$, we write $\psi \subseteq \psi'$ if $\dom(\psi) \subseteq \dom(\psi')$ and $\forall x \in \dom(\psi)$, $\psi(x) = \psi'(x)$.

We assume that a true unknown realization $\phi \in \Phi_V$ is drawn according to a known prior over $\Phi_V$. This prior is denoted by $p$ unless explicitly stated otherwise. An adaptive algorithm interactively selects elements and observes their states: At time $t$, an element $x_t \in V$ is selected and its state $\phi(x_t)$ is observed. The previous observations at time $t$ are represented by the partial realization $\psi_t \subseteq \phi$ whose domain is $\{x_1,\ldots,x_t\}$.
The utility of selecting a specific set of elements under a given true realization is measured using a \emph{utility function} $f:2^V \times \Phi_V \to \reals$, which maps each pair of a set of elements and a realization to a real value representing the utility when this set is selected under the given realization. The goal is to obtain a high utility while selecting few elements.

A \emph{policy} $\pi:\Psi_V \to V \cup \{\bot\}$ is a (possibly non-deterministic) map from partial realizations to elements that determines which element to select next, given the past observations represented by the input partial realization. $\pi(\psi) = \bot$ for a given $\psi \in \Psi_V$ indicates that $\pi$ terminates after observing $\psi$. For a policy $\pi$ and a realization $\phi \in \Phi_V$, we denote by $E(\pi, \phi) \subseteq V$ the (possibly random) set of elements that the policy $\pi$ selects through the entire run of the policy under realization $\phi$ until termination. The \emph{height} of a policy is the maximal number of elements it may select, for any true realization and any random bits. For an integer $k$, $\Pi_k$ denotes the set of all policies with height at most $k$.
For any policy $\pi$, a \emph{sub-policy} of $\pi$ is any policy that runs exactly like $\pi$, except that it possibly terminates earlier in some cases. 

The expected utility of a policy under a given prior $p$ for a given utility function $f$ is given by
\[
f_{\mathrm{avg}}(\pi, p) := \E[f(E(\pi, \phi), \phi)]. 
\]
Here and below, the expectation is over the true realization $\phi \sim p$ as well as the randomness of the policy.
The \emph{expected cost} of using a given policy is defined as the expected number of elements that it selects until termination, given by 
\[
c_{\mathrm{avg}}(\pi, p) := \E[|E(\pi, \phi)|].
\]

The \emph{expected marginal gain} of an element $v\in V$ under a given prior $p$ and a utility function $f$ with respect to a partial realization $\psi \in \Psi_V$ is the expected contribution to the utility value if $v$ is selected next, assuming that the elements selected so far and their observed states are as specified by $\psi \in \Psi_V$. Formally, the expected marginal gain is given by 
\[
\Delta^f_p(v\mid\psi) := \E[f(\{v\} \cup \dom(\psi), \phi) - f(\dom(\psi), \phi) \mid \phi \sim \psi].
\]
The expected marginal gain of a policy $\pi$ under $p$ and $f$, assuming the previous observations in $\psi$ is the expected contribution to the utility until termination, is denoted by
\[
\Delta^f_p(\pi \mid \psi) := \E[f(E(\pi, \phi) \cup \dom(\psi), \phi) - f(\dom(\psi), \phi) \mid \phi \sim \psi].
\]

\subsection{Properties of utility functions and policies}
\label{sec:Properties of utility functions and policies}
A utility function $f$ is \emph{adaptive monotone} if selecting an element is never harmful in expectation. Formally, $f$ is adaptive monotone if for any partial realization $\psi$ and any $v\in V\setminus \dom(\psi)$, \mbox{$\Delta^f_p(v\mid \psi) \geq 0$}.
A utility function $f$ is \emph{adaptive submodular} if selecting an element satisfies an expected diminishing-returns property, that is, if selecting the same element later never contributes more in expectation to the utility than selecting the same element earlier in the same run. Formally, a utility function is adaptive submodular if for any two partial realizations $\psi,\psi'$ such that $\psi \subseteq \psi'$ and for any $v\in V\setminus \dom(\psi')$, 
\[
  \Delta^f_p(v\mid \psi) \geq \Delta^f_p(v\mid \psi').
  \]
The \emph{adaptive submodularity ratio} of a utility function quantifies how close it is to being submodular.  For a given prior $p$ and two integers $n,k$, the adaptive submodularity ratio of $f$ is defined as
\[
\gamma^s_{n,k}(f, p) := \min_{\psi \in \Psi_V:|\psi| \leq n,\psi'\subseteq \psi,\pi\in\Pi_k} \frac{\sum_{v\in V} \P_{\Phi \sim p}[v\in E(\pi,\Phi)\mid \Phi\sim \psi']\Delta^f_p (v \mid \psi')}{\Delta^f_p(\pi\mid \psi')}.
\]
The adaptive submodularity ratio $\gamma^s_{n,k}(f, p)$ takes values in $[0,1]$, where a higher value indicates that a function is closer to being adaptive submodular, and a value of 1 indicates that the function is adaptive submodular. 

Previous works have studied \emph{greedy} or \emph{approximate greedy} policies, which, in each round, select an element that maximizes or approximately maximizes the expected marginal gain. For $\alpha \geq 1$, a policy is considered $\alpha$-approximate greedy if for any $\psi \in \Psi_V$ such that $\pi(\psi) \neq \bot$, 
\[
\Delta^f_p(\pi(\psi) \mid \psi) \geq 
\frac{1}{\alpha}  \max_{v\in V} \Delta^f_p(v\mid\psi).
\]
A greedy policy is a policy which is $1$-approximate greedy.
Denote by $\alpha_\pi(f,p)$ the smallest real value $\alpha$ such that $\pi$ is an $\alpha$-approximate greedy policy with respect to $f$ and $p$. We term this the \emph{greedy approximation ratio} of the policy.

\subsection{Previous approximation guarantees}
\label{sec:Previous approximation guarantees}

\cite{golovin2011adaptive} proved an approximation guarantee for the expected utility of an approximate greedy policy $\pi$
of height at most $l$, assuming that $f$ is an adaptive submodular utility function.
They proved that for any policy $\pi^*$ with height $k$, 
\begin{equation}\label{eq:gk11's theorem}
f_{\mathrm{avg}}(\pi,p) > \left(1-\exp(-\frac{l}{\alpha_{\pi}(f,p) \cdot k})\right)f_{\mathrm{avg}}(\pi^*, p).
\end{equation}

\cite{fujii2019beyond} considered the same setting, but allowed also utility functions that are not adaptive submodular, using the adaptive submodularity ratio of the utility function. They required the policy to be exactly greedy, thus not supporting $\alpha$-approximate greedy policies for $\alpha > 1$. They proved that for any policy $\pi^*$ with height $k$,
\begin{equation}\label{eq:fujii theorem}
f_{\mathrm{avg}}(\pi,p) \geq \left(1-\exp(-\frac{\gamma^s_{l,k}(f,p) \cdot l}{k})\right)\cdot f_{\mathrm{avg}}(\pi^*,p).
\end{equation}
This coincides with \eqref{gk11's theorem} for adaptive submodular functions ($\gamma^s_{l,k}(f,p)=1$) and greedy policies ($\alpha = 1$). 

\cite{esfandiari2021adaptivity} provided a guarantee that depends on the average costs of the policies. This contrasts with the results discussed above, which depend on the height of the two policies being compared. They considered only greedy policies and adaptive submodular functions. 
They proved that for a greedy policy $\pi$, if it terminates when the expected marginal gain of all remaining elements is below some fixed threshold, and if its average cost is $l$, then for any policy $\pi^*$, 
\begin{equation}\label{eq:esffavg}
f_{\mathrm{avg}}(\pi, p) > \left(1 - \exp\left(-\frac{l}{c_{\mathrm{avg}}(\pi^*, p)+1}\right)\right) \cdot f_{\mathrm{avg}}(\pi^*, p).
\end{equation}
They further proved an approximation guarantee for the minimum cost coverage
objective. This approximation guarantee holds for functions with a \emph{discrete covering
  property} first proposed by \cite{golovin2011adaptive}. A utility function
$f$ is said to satisfy the discrete covering property with parameters
$Q,\eta > 0$ if there exists a value $\eta \in (0,Q]$ such that for all
$\psi\in \Psi_V$ and $\phi \in \Phi_V$, if $f(\dom(\psi), \phi)>Q-\eta$, then
$f(\dom(\psi), \phi) = Q$.  \cite{esfandiari2021adaptivity} proved that for
any such function, and any greedy policy $\pi$, if $\pi^*$ is a policy that
achieves covering, that is, $f(E(\pi^*, \phi), \phi) = Q$ for all
$\phi \in \Phi_V$, then
\begin{equation}\label{eq:esfcavg}
c_{\mathrm{avg}}(\pi, p) \leq (c_{\mathrm{avg}}(\pi^*, p) + 1)  \log(\frac{|V|Q}{\eta}) + 1.
\end{equation}

This improves a previous bound by \cite{golovinKrause17arxiv}, which had a
squared logarithmic dependence.  \cite{golovin2011adaptive} studied
Generalized Binary Search (GBS), which is a specific strategy for Bayesian
active learning with binary labels, and defined a \emph{coverage utility
  function} for active learning that has the discrete covering property with
parameters $Q=1,\eta = \min_{\phi \in \Phi_V}p(\phi)$.  Plugging these
parameters into \eqref{esfcavg} may be disadvantageous if the minimal
probability assigned by $p$ is very small, due to the term $\log(1/\eta)$ in
the approximation factor. \cite{golovin2011adaptive,golovinKrause17arxiv} defined a modified prior $\pmod \propto \max\{p(\phi), 1/|\Phi_V|^2\}$, and derived an $O(\log^2(|\Phi_V|))$ approximation guarantee based on their squared-logarithmic approximation bound, for a greedy policy that uses a coverage utility function based on this prior.

\section{Main results}\label{sec:main}
In this section, we provide an overview of our main results.
In \secref{beta} below, we define a new policy parameter, which we term the \emph{maximal gain ratio}. This is the maximal ratio between the expected marginal gain of remaining elements at termination and the expected marginal gain of elements selected by the policy.
The maximal gain ratio of a policy $\pi$ given a utility function $f$ and a prior $p$ is denoted by $\beta_\pi(f,p)$. Our approximation guarantees, provided below, are parameterized by $\beta_\pi(f,p)$, and are stronger when $\beta_\pi(f,p)$ is smaller. 
It is therefore desirable to have $\beta_\pi(f,p)$ as small as possible. In \secref{beta}, we prove that the value of $\beta_\pi(f,p)$ is never larger than that of $\alpha_\pi(f,p)$, the greedy approximation ratio of the policy, and that it can actually be significantly smaller. This shows that a guarantee that uses $\beta_\pi(f,p)$ is strictly stronger than the same guarantee with $\alpha_\pi(f,p)$.
We prove the following properties:
\begin{itemize}
\item For all policies $\pi$, utility functions $f$ and priors $p$, $\beta_\pi(f,p) \leq \alpha_\pi(f,p)$.
\item $\beta_\pi$ can have a value of $1$ even if the policy is not greedy. This contrasts with $\alpha_\pi$, which is equal to $1$ if and only if the policy $\pi$ is greedy.
\item $\beta_\pi$ can be arbitrarily close to zero. This is in contrast to $\alpha_\pi$, which is bounded below by $1$.
\end{itemize}

Having defined the maximal gain ratio, we provide new approximation guarantees that use this parameter, instead of the greedy approximation ratio or an assumption that a policy is greedy, as in previous works. In addition, these approximation guarantees support multiple regimes that were previously studied separately: functions that are not necessarily adaptive submodular, policies that are not necessarily greedy, and guarantees for the minimum cost coverage objective for functions with the discrete covering property. 

Our first approximation guarantee concerns utility maximization. 
\begin{theorem}\label{thm:esfandiari favg}
Let $f$ be a non-negative utility function which is adaptive monotone with respect to a prior $p$.
Let $\pi$ be a policy with height $n$ that terminates when the expected marginal gain of all remaining elements is below some fixed threshold and has an average cost of $l$.\footnote{We provide a formal version of this condition in \appref{New guarantees for nearly adaptive submodular functions}.} Let $\pi^*$ be a policy with height $k$. Then
\[
f_{\mathrm{avg}}(\pi, p) \geq \left(1 - \exp\left(-l\bigg/\bigg(\frac{\beta_\pi(f,p)  }{\gamma^s_{n,k}(f,p)}\cdot c_{\mathrm{avg}}(\pi^*, p)+1\bigg)\right)\right) \cdot f_{\mathrm{avg}}(\pi^*, p).
\]
\end{theorem}
This result includes as a special case the conditions of \eqref{esffavg}, which require $f$ to be adaptive submodular and $\pi$ to be greedy. In this case, $\gamma^s_{n,k}(f,p) = 1$ and $\beta_\pi(f,p) \leq 1$. Thus, even in this special case, the theorem provides a smaller approximation factor whenever $\beta_\pi(f,p) < 1$. Moreover, the same guarantee as \eqref{esffavg} is obtained also for non-greedy policies if they have $\beta_\pi(f,p) = 1$.

Our next approximation guarantee considers the minimum cost coverage objective.
\begin{theorem}\label{thm:esfandiari cavg}
Let $f$ be a non-negative utility function which is adaptive monotone with respect to a prior $p$, such that $f$ has the discrete covering property with parameters $Q,\eta$. 
Let $\pi$ be a policy with height $n$.
Let $\pi^*$ be a policy with height $k$ such that $f(E(\pi^*, \phi), \phi) = Q$ for all $\phi \in \Phi_V$.
Then,
\[
c_{\mathrm{avg}}(\pi, p) \leq \left(\frac{\beta_\pi(f,p)}{\gamma^s_{n,k}(f,p)}\cdot c_{\mathrm{avg}}(\pi^*, p) + 1\right) \cdot \log(\frac{nQ}{\eta}) + 2.
\]
\end{theorem}
Considering again the special case in which $f$ is adaptive submodular and $\pi$ is greedy, we recover \eqref{esfcavg} up to an additive term of $1$, noting that $n \leq |V|$. Similarly to \thmref{esfandiari favg}, here too this special case is in fact broader and can provide stronger guarantees than the corresponding inequality \eqref{esfcavg}.
The two theorems above are proved in \appref{New guarantees for nearly adaptive submodular functions}.

Lastly, we provide approximation guarantees that allow providing a stronger guarantee for priors with small probabilities, using the modified prior $\pmod$ of \cite{golovin2011adaptive}. We prove our results for a more general setting with general utility functions, and apply them to the coverage utility function proposed by \cite{golovin2011adaptive} for Bayesian active learning. We conclude that
when using a greedy policy with a coverage utility function based on the modified prior, denoted $f_{\pmod}$, we have the following average cost approximation guarantee for any policy $\pi$, compared to the optimal policy $\pi^*$.
\begin{equation}\label{eq:coverage}
c_\mathrm{avg}(\pi, p) \leq 2 \big(\beta_\pi(f_{\pmod},\pmod) \cdot(c_{\mathrm{avg}}(\pi^*, p) + 1) + 1\big) \log(2|\Phi_V|^2n) + 4.
\end{equation}
In particular, if $\pi$ is greedy with respect to $\pmod$, then $\beta_\pi(f_{\pmod},\pmod) \leq 1$, and it could be significantly smaller, as discussed above.
We provide the necessary definitions and prove \eqref{coverage} in \secref{modified}.

Our results thus provide a comprehensive and improved set of approximation guarantees that encompass a wide range of problem parameters, and highlight the usefulness of the new maximal gain ratio parameter.  In the next section, we define the maximal gain ratio and prove its properties. 

\section{The Maximal Gain Ratio}\label{sec:beta}
In this section, we define the maximal gain ratio and prove its properties as listed in \secref{main}.
We start with necessary definitions.
Given a utility function $f$, a prior $p$ and a policy $\pi$, let $\tau \geq 0$ and $\rho \in [0,1]$ be real numbers.
We denote by $\pi^{\tau, \rho}$ a sub-policy of $\pi$ that with probability $\rho$ terminates if the expected marginal gain of every single remaining element is strictly smaller than $\tau$, and otherwise (with probability $1-\rho$), terminates if the expected marginal gain of every single remaining element is at most $\tau$. We call $\tau$ a \emph{threshold} and $\rho$ a \emph{tie-break probability}. To define the maximal gain ratio, we require $\tau,\rho$ such that $c_{\mathrm{avg}}(\pi^{\tau,\rho},p)$ is equal to a given integer. \cite{esfandiari2021adaptivity} studied these types of sub-policies when $\pi$ is greedy, and provided a proof sketch for the existence of such $\tau,\rho$ in this context. We provide here a full unconditional existence proof for any policy $\pi$, and prove in addition that there is a unique such sub-policy for each possible integer.

\begin{lemma}\label{lem:any policy has a unique sub-policy of size i of that kind}
Assume a utility function $f$, a policy $\pi$, and a prior $p$.
For any integer $i \leq c_\mathrm{avg}(\pi, p)$, there exist a threshold $\tau_i \geq 0$ and a tie-break probability $\rho_i \in [0,1]$ such that $c_{\mathrm{avg}}(\pi^{\tau_i, \rho_i}, p) = i$.
Moreover, all such pairs $\tau_i,\rho_i$ induce the same policy.
\end{lemma}
\begin{proof}
For a threshold $\tau$, denote $\mu^{\tau} := c_\mathrm{avg}(\pi^{\tau, 1}, p)$.
Define an equivalence relation between thresholds as follows: for every $\tau, \tau'$, say that $\tau \equiv \tau'$ if $\mu^{\tau} = \mu^{\tau'}$.
Since the sets of elements and the set of possible states are finite, the set of all the marginal gains for a given prior is also finite. Hence, there is a finite number of equivalence classes. Denote this number $s$.
Fix some representatives of the equivalence classes and denote them by $\tau^1 > ... > \tau^s$.
Note that $\mu^{\tau^s} = c_\mathrm{avg}(\pi, p)$. Define $\mu^{\tau^0} = 0$, and note that the sequence $(\mu^{\tau^{j}})_{j \in \{0,\ldots,s\}}$ is strictly increasing.
Let $i \leq c_\mathrm{avg}(\pi, p)$. Clearly, $i \leq \mu^{\tau_s}$.
Let $j$ be the smallest index in $[s]$ such that $i\leq \mu^{\tau^{j}}$. Denote $\tau_i := \tau^j$ and
\[
\rho_i := \frac{i - \mu^{\tau^{j-1}}}{\mu^{\tau^{j}} - \mu^{\tau^{j-1}}}.
\]
Note that $\rho_i \in [0,1]$.
Consider the policy $\pi^{\tau_i, \rho_i}$. By the definition of thresholds and tie-break probabilities, it holds that
\begin{equation}\label{eq:cavg calculation}
c_{\mathrm{avg}}(\pi^{\tau_i, \rho_i}, p) = (1-\rho_i)\mu^{\tau^{j-1}} + \rho_i \mu^{\tau^{j}} = \mu^{\tau^{j-1}} + \rho_i(\mu^{\tau^{j}} - \mu^{\tau^{j-1}}) = i,
\end{equation}
Where the last equality follows from the definition of $\rho_i$. 
This proves the existence of $(\tau_i,\rho_i)$, as claimed in the first part of the lemma. 

For the second part of the lemma, we show that any such pair would induce the same policy. Let $\tau'$ and $\rho'$ be such that $c_{\mathrm{avg}}(\pi^{\tau', \rho'}, p) = i$. First, note that one of $\pi^{\tau', \rho'}$ and $\pi^{\tau_i, \rho_i}$ must be a sub-policy of the other: If $\tau' < \tau_i$, then $\pi^{\tau_i, \rho_i}$ is a sub-policy of $\pi^{\tau', \rho'}$, since the former always terminates before or with the latter. If $\tau_i < \tau'$, then the symmetric claim holds. Lastly, if $\tau' = \tau_i$, then we may assume w.l.o.g~that the policy with the larger $\rho$ never terminates in cases where the other policy continues, by assuming a coordinated draw of random bits.

Next, assume for contradiction that $\pi^{\tau', \rho'} \ne \pi^{\tau_i, \rho_i}$.
Therefore, there exists at least one partial realization $\psi$ and random bits such that immediately after observing it, $\pi^{\tau', \rho'}$ terminates, while $\pi^{\tau_i, \rho_i}$ selects an element. Since no partial realization satisfies the opposite case, it follows that $i = c_{\mathrm{avg}}(\pi^{\tau', \rho'}, p) < c_{\mathrm{avg}}(\pi^{\tau_i, \rho_i}, p) = i$, a contradiction.
Thus, in all cases, $\pi^{\tau', \rho'} = \pi^{\tau_i, \rho_i}$, as claimed.
\end{proof}

Given a policy $\pi$ and an integer $i\leq c_\mathrm{avg}(\pi, p)$, denote by $\pi_i$ the unique sub-policy of $\pi$ using the threshold and tie-break construction and has an average cost of $i$.
Denote by $\pi_0$ a policy that terminates before selecting any element.
Denote by $\Delta^{u}_{\pi,i}$ the maximal expected marginal gain of remaining elements when $\pi_i$ terminates in any of its possible paths. Formally, $\Delta^{u}_{\pi,i}$ is the smallest real value $u$ such that
\[
  \P[\text{For $t \in \nats$ such that }\pi_i \text{ terminates at time $t$, } \forall v \in V, \Delta_p^f(v\mid \psi_t) \leq u] = 1.
\]
Denote by $\Delta^{l}_{\pi,i}$ the smallest possible expected marginal gain of elements that $\pi_i$ selects at any time during its run. Formally,
this is the largest real value $u$ such that
\[
  \P[\text{For all }t\text{ until $\pi_i$ terminates}, \Delta_p^f(\pi_i(\psi_t)  \mid\psi_t) \geq u] = 1.
\]

Equipped with these definitions, we can now define the maximal gain ratio.
\begin{definition}[maximal gain ratio]\label{def:maximal gain ratio}
The \emph{maximal gain ratio} of a policy $\pi$ with respect to a utility function $f$ and a prior $p$ is  $\beta_\pi(f,p) := \max_{i\in \nats, i \leq c_{\mathrm{avg}(\pi, p)}} \Delta^{u}_{\pi,i}/\Delta^{l}_{\pi,i}$.
\end{definition}
The maximal gain ratio is thus based on comparing the possible gain of elements that were not selected, to the gain from selected elements.  We now show that the maximal gain ratio never larger than the greedy approximation parameter, and is sometimes strictly smaller. This shows that \thmref{esfandiari favg}, \thmref{esfandiari cavg} and \eqref{coverage}, are stronger due to their use of $\beta_\pi$ than if they had relied on $\alpha_\pi$ instead.

First, we show that the maximal gain ratio of an approximate greedy policy is always upper bounded by the greedy approximation parameter of the same policy.
\begin{theorem}\label{thm:maximal gain ratio bounded by approximation factor}
For any prior $p$, utility function $f$, and policy $\pi$, we have $\beta_\pi(f,p) \leq \alpha_\pi(f,p)$.
\end{theorem}
\begin{proof}
Let $\pi$ be a policy such that $\alpha_\pi(f,p) < \infty$, and denote $\alpha := \alpha_\pi(f,p)$.
Let $i \leq c_{\mathrm{avg}}(\pi, p)$.
Let $\tau_i$ and $\rho_i$ be a threshold and a tie-break probability, respectively, such that $\pi_i:=\pi^{\tau_i,\rho_i}$.
By the definition of $\pi^{\tau_i, \rho_i}$, we have that $\tau_i$ is an upper bound on the expected marginal gain of all the elements that remain when $\pi^{\tau_i, \rho_i}$ terminates. Therefore, $\Delta^{u}_{\pi,i} \leq \tau_i$.

At all times during the run of $\pi_i$ up to the time of termination, there exists some element with an expected marginal gain of at least $\tau_i$. Otherwise, by definition of $\pi_i=\pi^{\tau_i,\rho_i}$, the run would have terminated earlier. Therefore, since $\pi$ is an $\alpha$-approximate greedy policy, the expected marginal gain of the elements that $\pi_i$ selects at any time during its run is at least $\tau_i/\alpha$.
It follows that $\Delta^{l}_{\pi,i} \geq \tau_i/\alpha$.
Therefore,
\[
\beta_\pi(f,p) = \max_{i \leq c_{\mathrm{avg}(\pi, p)}} \Delta^{u}_{\pi,i}/\Delta^{l}_{\pi,i} \leq
\max_{i \leq c_{\mathrm{avg}(\pi, p)}} \frac{\tau_i}{\tau_i/\alpha} = \alpha.
\]
This completes the proof.
\end{proof}
We now show that there are cases in which $\beta_\pi$ is strictly smaller than $\alpha_\pi$. First, we show that an approximate greedy policy that is not greedy can still have $\beta_\pi(f,p) = 1$. 
\begin{theorem}\label{thm:non-greedy policy with maximal gain ratio 1}
For any integer $k \geq 3$, there exist a policy $\pi$ with height $k$, a prior $p$ and a utility function $f$ such that $\beta_\pi(f,p)=1$ while $\alpha_\pi(f,p) = 2$.
\end{theorem}
\begin{proof}
For a given integer $k \geq 3$, let $V = \{v_1,\ldots,v_k\}$, and let $Y = \{0, 1\}$.
For any $i \in [k]$, we denote $v_{[i]} := \{v_1,\ldots,v_i\}$.
We define a utility function $f$ as follows. For a realization $\phi$ and a subset $A \subseteq V$, let $\pi$ be a policy that selects $v_i$ in the $i$'th step of its run and terminates immediately after selecting $v_k$.
By the properties of $\Delta^f_{p}(v \mid \psi)$, if the observed partial realization is $\psi$ such that  $\dom(\psi) = \emptyset$ or $\dom(\psi) = v_{[k-1]}$, then the next element $v$ that $\pi$ selects satisfies $\Delta^f_{p}(v \mid \emptyset) = 1/k = \max_{v \in V} \Delta^f_{p}(v \mid \psi).$
Additionally, if $\dom(\psi) = v_{[i-1]}$ for some $i \in \{2,\ldots,k-1\}$, then
\[
\Delta^f_{p}(v_i \mid \psi) = \frac{1}{k} = \frac{1}{2} \cdot \frac{2}{k} = \frac{1}{2} \cdot \Delta^f_{p}(v_{i+1} \mid \psi) = \frac{1}{2} \cdot \max_{v \in V} \Delta^f_{p}(v \mid \psi).
\]
Therefore, $\alpha_\pi(f,p) = 2$.

On the other hand, we now show that $\beta_\pi(f,p) = 1$.
Let $i \leq c_{\mathrm{avg}}(\pi,p)$.
First, we show that for $\tau = 1/k$ and $\rho = i/k$, it holds that $\pi_i = \pi^{\tau, \rho}$. Due to the value of $\tau$, with probability $\rho$, $\pi^{\tau, \rho}$ selects all of the elements of $V$ and terminates. Otherwise (with probability $1-\rho$) $\pi^{\tau, \rho}$ terminates before selecting any element. Therefore,
\[
c_{\mathrm{avg}}(\pi^{\tau, \rho},p) = \rho \cdot k + (1-\rho) \cdot 0 = \frac{i}{k} \cdot k  = i.
\]
By \lemref{any policy has a unique sub-policy of size i of that kind}, $\pi_i = \pi^{\tau, \rho}$.
Now, the largest possible expected marginal gain of remaining elements when $\pi_i$ terminates is $1/k$. Hence, $\Delta^{u}_{\pi,i} = 1/k$.
The smallest possible expected marginal gain of any element selected by $\pi_i$ is also $1/k$.
Hence, $\Delta^{l}_{\pi,i} = 1/k$. Therefore,
\[
\beta_\pi(f,p) = \max_{i \leq c_{\mathrm{avg}}(\pi,p)} \Delta^{u}_{\pi,i}/\Delta^{l}_{\pi,i} = \max_{i \leq c_{\mathrm{avg}}(\pi,p)} \frac{1/k}{1/k} = 1.
\]
This completes the proof.
\end{proof}Lastly, we show that there exist policies $\pi$ that are greedy with respect to a utility function $f$ and a prior $p$ that have an arbitrarily small $\beta_\pi(f,p)$. Thus, guarantees based on $\beta_\pi(f,p)$ are strictly stronger than guarantees based only on the fact that a policy is greedy.
\begin{theorem}\label{thm:beta pi is arbitrarily small}
For any integer $k$ and any $\epsilon > 0$, there exist a prior $p$, a utility function $f$ and a greedy policy $\pi$ with respect to $f$ that selects $k$ elements, such that $\beta_\pi(f,p) = \epsilon$.
\end{theorem}
\begin{proof}
Let $k$ be an integer and let $\epsilon > 0$.
Let $V = \{v_1, v_2,\ldots,v_k\}$, $Y = \{0,1\}$.
We define a utility function $f$ as follows. For any realization $\phi$ and any subset $A \subseteq V$, $f(A, \phi) = \sum_{i=0}^{|A|} \epsilon^{i}$.
Let the prior $p$ be a uniform distribution over the realizations.
By the definition of $f$, for any partial realization $\psi$, it holds that $\Delta^f_{p}(v \mid \psi) = \epsilon^{|\psi|+1}$.

Let $\pi$ be a policy that selects $v_i$ in the $i$'th step of its run and terminates immediately after selecting $v_k$. For any partial realization $\psi$, all of the elements $v \in V$ has the same value of $\Delta^f_p(v \mid \psi)$. Therefore $\pi$ is greedy with respect to $f$.

Recall that $\beta_\pi(f,p) = \max_{i \leq c_{\mathrm{avg}(\pi, p)}} \Delta^{u}_{\pi,i}/\Delta^{l}_{\pi,i}$. To bound $\beta_\pi(f,p)$, consider the possible values of $i \leq c_{\mathrm{avg}(\pi, p)} = k$. For $i=k$, the largest possible expected marginal gain of remaining elements when $\pi_n$ terminates is $0$. Therefore, $\Delta^{u}_{\pi,k} = 0$ and so $\Delta^{u}_{\pi,k}/\Delta^{l}_{\pi,k} = 0$.
For $i \leq k-1$, the largest possible expected marginal gain of remaining elements when $\pi_i$ terminates is $\epsilon^{i+1}$. Therefore, $\Delta^{u}_{\pi,i} = \epsilon^{i+1}$.
The smallest possible expected marginal gain of elements that $\pi_i$ selects at any time during its run is $\epsilon^{i}$. Therefore, $\Delta^{l}_{\pi,i} = \epsilon^{i}$. Hence, $\Delta^{u}_{\pi,i}/\Delta^{l}_{\pi,i} = \epsilon^{i+1}/\epsilon^{i} = \epsilon$. It follows that
$\beta_\pi(f,p) := \max_{i \leq c_{\mathrm{avg}(\pi, p)}} \Delta^{u}_{\pi,i}/\Delta^{l}_{\pi,i} = \max \{0, \epsilon\} = \epsilon$, as claimed
\end{proof}
We note that it has been shown in previous works \citep{conforti1984submodular, vondrak10} that for some utility functions, stronger approximation guarantees can be provided for greedy policies than those that can be obtained for general submodular functions, based on the curvature property of the function.  It is thus helpful to observe that guarantees based on curvature cannot
be used to provide stronger approximation guarantees in the example provided in the proof above. Curvature-based guarantees are stronger if the utility function
is closer to being modular (additive). In contrast, in the example above, the
function is far from being modular, as adding the same element later in the
run results in a significantly lower gain. More precisely, the curvature
property takes values in $[0,1]$ for submodular functions, and results concerning curvature are
stronger than results for general submodular functions if the curvature is smaller. However, the curvature of the function in the proof above is
$1-\epsilon^{k-1}$, making it a poor candidate to gain from
curvature-based bounds.

The properties proved above for the maximal gain ratio imply that \thmref{esfandiari favg} and \thmref{esfandiari cavg} are strictly stronger than those in previous works, even for adaptive submodular utility functions and greedy policies. The proofs of these theorems are provided in \appref{New guarantees for nearly adaptive submodular functions}.

\section{Using a modified prior}\label{sec:modified}
In this section, we present an approximation guarantee that uses the maximal gain ratio and the adaptive submodularity ratio for the modified prior, and use it to conclude \eqref{coverage}, which provides an approximation guarantee that does not depend on the smallest probability in the true prior.

\begin{theorem}\label{thm:modified}
Let $f$ be a non-negative utility function which is adaptive monotone with respect to a prior $p$, such that $f$ has the discrete covering property with parameters $Q,\eta$. 
Let $\pi$ be a policy with height $n$.
Let $\pi^*$ be a policy with height $k \leq |\Phi_V|$ such that $f(E(\pi^*, \phi), \phi) = Q$ for all $\phi \in \Phi_V$.
Let $p'(\phi) \propto \max\{p(\phi), 1/|\Phi_V|^2\}$ be the modified prior of $p$. Then
\[
c_\mathrm{avg}(\pi, p) \leq 2 \left(\frac{\beta_\pi(f,p') }{\gamma^s_{n,k}(f, p')} (c_{\mathrm{avg}}(\pi^*, p) + 1) + 1\right) \log(\frac{nQ}{\eta}) + 4.
\]
\end{theorem}
This result improves over the previous guarantees of \cite{golovin2011adaptive,golovinKrause17arxiv} for the modified prior, since the latter had a square logarithmic factor and only applied to greedy policies.
The proof of this theorem is provided in \appref{modified}.

Before deriving \eqref{coverage}, we provide background on the coverage utility function proposed by \cite{golovin2011adaptive} and its application to active learning. In the context of active learning, the observations in combinatorial maximization setting represent labeled examples. The coverage function measures the probability mass of realizations that have been disqualified due to the observations in the observed partial realization. To define the coverage function formally, denote the \emph{version space} of $\psi \in \Psi$ by $\VS(\psi) := \{\phi \in \Phi_V \mid \phi \sim \psi\}.$
For $\cH \subseteq \Phi$, denote $p(\cH) := \sum_{\phi \in \cH} p(\phi)$. For $A \subseteq V$, $\phi \in \Phi$, denote $\VS_{A,\phi} = \VS(\{(x,\phi(x)) \mid x \in A\}).$
\begin{definition}[coverage utility function]
  For a prior $p$ over $\Phi_V$, the coverage utility function is denoted $f_p: 2^V \times \Phi_V \to [0,1]$ and defined such that for any realization $\phi$ and any set $A \subseteq V$,
$f_p(A,\phi) := 1 - p(\VS_{A,\phi}) + p(\phi)$.
\end{definition}

\cite{golovin2011adaptive} proved that for any prior $p$, $f_p$ is adaptive monotone and adaptive submodular with respect to $p$. The latter implies that $\gamma^s_{n,k}(f_{p'}, p')=1$. In addition, note that if $\VS_{A,\phi} = \{\phi\}$, that is, the true realization is fully identified by observing the labels of the examples in $A$, then  $f_p(A,\phi) = 1$, which is the function's maximal possible value. Thus, in this case we can set $Q=1$ in \thmref{modified}. In addition, for $f_p$, $\eta = \min_{\phi \in \Phi_V} p(\phi)$. When using $f_{p'}$ with the modified prior $p'$, we get
\[
\eta = \min_{\phi \in \Phi_V} p'(\phi) \geq \frac{1/|\Phi_V|^2}{\sum_{\phi \in \phi_V}\max \{p(\phi), 1/|\Phi_V|^2\}}\geq \frac{1/|\Phi_V|^2}{1+1/|\Phi_V|} \geq \frac{1}{2|\Phi_V|^2}.
\]
By substituting these in \thmref{modified}, we conclude that for any policy $\pi^*$ with height at most $|\Phi_V|$, \eqref{coverage} holds.
We have left to show that $\pi^*$ can be set to an optimal policy, by proving that there exists an optimal policy with height at most $|\Phi_v|$. While this makes intuitive sense, we formally state and prove it for completeness in \lemref{optimal policy always eliminates realization} in \appref{modified}.
This completes the derivation of \eqref{coverage}.

\section{Conclusion}\label{sec:Conclusions}
In this work, we provided new approximation guarantees for adaptive combinatorial maximization. We identified a new policy parameter, the maximal gain ratio, which is less restrictive than the greedy approximation parameter, and showed that using this parameter can lead to stronger approximation guarantees. Moreover, our guarantees are comprehensive, simultaneously supporting the new policy parameter as well as near-submodular utility functions, and include both maximization under a cardinality constraint and a minimum cost coverage guarantee. In addition, we provided an improved approximation guarantee for a modified prior, which is crucial for active learning guarantees that do not depend on the smallest probability in the prior. 

The new maximal gain ratio parameter sheds new light on the properties a policy needs to have to support general approximation guarantees for adaptive combinatorial maximization. This opens an exciting direction of research on the relationship between policy properties and the approximations that can be obtained based on them.

\acks{Resources used in preparing this research were provided, in part, by the Province of Ontario, the Government of Canada through CIFAR, and companies sponsoring the Vector Institute; see \url{https://vectorinstitute.ai/partnerships/current-partners/}.}

\bibliography{biblio}

\appendix
\section{Proving the approximation guarantees for the given prior}
\label{app:New guarantees for nearly adaptive submodular functions}
\label{sec:Proving the approximation guarantees}
In this section, we prove \thmref{esfandiari favg} and \thmref{esfandiari cavg}, stated in \secref{main}.
To prove \thmref{esfandiari favg}, we first prove the following lemma, which provides a lower bound for the difference between the average utility of two sub-policies. 
\begin{lemma}\label{lem:avg utility difference lower bound}
Let $\pi$ be a policy, $p$ be a prior, and $f$ be a utility function.
For any integer $i \leq c_{\mathrm{avg}}(\pi, p)$, define  $\Delta_i := f_{\mathrm{avg}}(\pi_i, p) - f_{\mathrm{avg}}(\pi_{i-1}, p)$.
Then $\Delta_i \geq \Delta^{l}_{\pi,i}.$
\end{lemma}
\begin{proof}
Let $i \leq c_{\mathrm{avg}}(\pi, p)$.
Let $\tau_i,\tau_{i-1},\rho_i,\rho_{i-1}$ be parameters such that $\pi_i = \pi^{\tau_i, \rho_i}$ and $\pi_{i-1} = \pi^{\tau_{i-1}, \rho_{i-1}}$. These exist according to \lemref{any policy has a unique sub-policy of size i of that kind}.  Observe that necessarily, $\tau_i \leq \tau_{i-1}$. This is because if $\tau_{i-1} < \tau_i$, then the termination condition of $\pi_{i}$ is stronger than that of $\pi_{i-1}$, implying that $i = c_{\mathrm{avg}}(\pi_i, p)\leq c_{\mathrm{avg}}(\pi_{i-1}, p) = i-1$, a contradiction.

Observe that $\pi_{i-1}$ must be a sub-policy of $\pi_i$: If $\tau_i < \tau_{i-1}$, then $\pi_{i-1}$ is a sub-policy of $\pi_i$, since the former always terminates before or with the latter. If $\tau_i = \tau_{i-1}$, then $\rho_i < \rho_{i-1}$. In this case, $\pi_i$ never terminates in cases where $\pi_i$ continues, by assuming a coordinated draw of random bits.

For a set of elements $S$, a realization $\phi$ and a partial realization $\psi$, we define $\Delta^f_p(S \mid \psi) := \E[f(S \cup \dom(\psi), \phi) - f(\dom(\psi), \phi) \mid \phi \sim \psi]$.
Let $A_{\psi, i, S}$ to be the event that the following hold:
\begin{enumerate}
\item The policy $\pi_{i-1}$ selects exactly all of the elements of $\dom(\psi)$,
\item The observations in $\psi$ are realized, and
\item the set of elements selected by $\pi_i$ coincides with $\dom(\psi)$ at some point in its run, and immediately after that, $\pi_i$ selects exactly all of the elements of $S$ and terminates.
\end{enumerate}
Since $\pi_{i-1}$ is a sub-policy of $\pi_i$, we have
\begin{align*}
\Delta_i &= f_{\mathrm{avg}}(\pi_i, p) - f_{\mathrm{avg}}(\pi_{i-1}, p)\\
&= \sum_{\psi \in \Psi_V} \sum_{S \subseteq V} \P[A_{\psi, i, S}] \Delta^f_p(S \mid \psi)\\
&\geq \sum_{\psi \in \Psi_V} \sum_{S \subseteq V} \P[A_{\psi, i, S}] \cdot |S| \cdot \Delta^{l}_{\pi,i}\\
&= (c_{\mathrm{avg}}(\pi_i, p) - c_{\mathrm{avg}}(\pi_{i-1}, p)) \cdot \Delta^{l}_{\pi,i}\\
&= (i - (i-1)) \cdot \Delta^{l}_{\pi,i} = \Delta^{l}_{\pi,i},
\end{align*}
where the inequality follows from the definition of $\Delta^{l}_{\pi,i}$. This proves the claim.
\end{proof}
Using this lemma, we can now prove \thmref{esfandiari favg}.

The statement of  \thmref{esfandiari favg} includes a condition that the policy $\pi$ terminates when the expected marginal gain of all remaining elements is below some fixed threshold and has an average cost of $l$. Formally, we require $\pi$ to be equal to $\pi'_l$ for some policy $\pi'$ (see \lemref{any policy has a unique sub-policy of size i of that kind} and the definition of $\pi_i$ thereafter).

\begin{proof}[of \thmref{esfandiari favg}]
For any realization $\phi$ and any element $v \in V$, denote $\mathbbm{1}_{\phi}^v := \one[v \in E(\pi^*, \phi)]$.
For any partial realization $\psi$ such that $|\psi| \leq n$, by the definition of $\gamma^s_{n,k}(f,p)$, we have
\begin{align*}
\gamma^s_{n,k}(f,p) \cdot \Delta^f_p(\pi^* \mid \psi) &\leq \sum_{v \in V} \P_{\Phi \sim p}[v \in E(\pi^*, \Phi) \mid \Phi \sim \psi]  \Delta^f_p(v\mid\psi)\\
  &= \sum_{v \in V} \E[\mathbbm{1}[v \in E(\pi^*, \Phi)] \mid \Phi \sim \psi] \Delta^f_p(v\mid\psi) \\
  &\leq \E[\,|E(\pi^*, \Phi)|\,]\, \mid \Phi \sim \psi] \max_{v \in V} \Delta^f_p(v\mid\psi).
\end{align*}
Therefore, 
\begin{equation}\label{eq:policy gain upper bound}
\Delta^f_p(\pi^* \mid \psi) \leq \frac{\E[\,|E(\pi^*, \Phi)|\,]\, \mid \Phi \sim \psi] \max_{v \in V} \Delta^f_p(v\mid\psi)}{\gamma^s_{n,k}(f,p)}.
\end{equation}

For any integer $i \leq c_{\mathrm{avg}}(\pi_i, p)$, let $\Delta_i$ be defined as in \lemref{avg utility difference lower bound}. 
For two policies $\pi',\pi''$, we denote by $\pi'@\pi''$ the policy that first runs $\pi'$ and then runs $\pi''$ from the beginning, without taking into account any information collected during the execution of $\pi'$. Note that $E(\pi'@\pi'', \phi) = E(\pi', \phi) \cup E(\pi'', \phi)$ for every realization $\phi$.
In particular, it holds that $E(\pi_i@\pi^*, \phi) = E(\pi_i, \phi) \cup E(\pi^*, \phi)$ for any realization $\phi$. Therefore, by the adaptive monotonicity of $f$, we have $f_{\mathrm{avg}}(\pi^*, p) \leq f_{\mathrm{avg}}(\pi_i@\pi^*, p)$.
For any $\psi \in \Psi_V$, we denote by $A_{\psi}$ the event that $\pi_i$ observes $\psi$ and then terminates.
For $i \in \{0,1,\ldots,l\}$ we define  $\Delta_i^* = f_{\mathrm{avg}}(\pi^*, p) - f_{\mathrm{avg}}(\pi_i, p)$.
For $i \in [l]$,
\begin{align*}
\Delta_i^* &= f_{\mathrm{avg}}(\pi^*, p) - f_{\mathrm{avg}}(\pi_i, p)\\
&\leq f_{\mathrm{avg}}(\pi_i@\pi^*, p) - f_{\mathrm{avg}}(\pi_i, p)\\
&= \sum_{\psi \in \Psi_V} \P[A_{\psi}]\Delta^f_p(\pi^* \mid \psi)\\
&\leq \frac{\sum_{\psi \in \Psi_V} \P[A_{\psi}]\E[\,|E(\pi^*, \Phi)|\,]\, \mid \Phi \sim \psi] \max_{v \in V} \Delta^f_p(v\mid\psi)}{\gamma^s_{n,k}(f,p)},
\end{align*}
where the last inequality is by \eqref{policy gain upper bound}.
By the definition of $\Delta^{u}_{\pi,i}$, it holds that $\max_{v \in V} \Delta^f_p(v\mid\psi) \leq \Delta^{u}_{\pi,i}.$
By the definition of $\beta_\pi(f,p)$, we have $\Delta^{u}_{\pi,i} \leq \beta_\pi(f,p) \Delta^{l}_{\pi,i} \leq \beta_\pi(f,p) \Delta_i,$
where the last inequality follows from \lemref{avg utility difference lower bound}.
Therefore, $\max_{v \in V} \Delta^f_p(v\mid\psi) \leq \beta_\pi(f,p)  \Delta_i.$
It follows that
\[
\Delta_i^* \leq \frac{\sum_{\psi \in \Psi_V} \P[A_{\psi}]\E[\,|E(\pi^*, \Phi)|\,]\, \mid \Phi \sim \psi] \cdot \beta_\pi(f,p)  \Delta_i}{\gamma^s_{n,k}(f,p)}.
\]
In addition,
\[
\sum_{\psi \in \Psi_V} \P[A_{\psi}]\E[\,|E(\pi^*, \Phi)|\,]\, \mid \Phi \sim \psi] = \E[\,|E(\pi^*, \Phi)|\,]\,] = c_{\mathrm{avg}}(\pi^*, p).
\]
Let $c^* = c_{\mathrm{avg}}(\pi^*, p)$.
Hence, 
\[
\Delta_i^* \leq \frac{\beta_\pi(f,p)  \cdot c^*}{\gamma^s_{n,k}(f,p)} \cdot \Delta_i.
\]
By the definitions of $\Delta_i$ and $\Delta_i^*$, it holds that 
\begin{align*}
\Delta_i &= f_{\mathrm{avg}}(\pi_i, p) - f_{\mathrm{avg}}(\pi_{i-1}, p)\\
&= f_{\mathrm{avg}}(\pi^*, p) - f_{\mathrm{avg}}(\pi_{i-1}, p) -(f_{\mathrm{avg}}(\pi^*, p) - f_{\mathrm{avg}}(\pi_i, p)) = \Delta_{i-1}^* - \Delta_i^*.
\end{align*}
Therefore,
\[
\Delta_i^* \leq \frac{\beta_\pi(f,p)  \cdot c^*}{\gamma^s_{n,k}(f,p)}(\Delta_{i-1}^* - \Delta_i^*).
\]
Rearranging, we conclude that 
\[
\Delta_i^* \leq \left(1 - \frac{1}{\beta_\pi(f,p) \cdot c^*/\gamma^s_{n,k}(f,p)+1}\right) \cdot \Delta_{i-1}^*.
\]
By induction, for any integer $l$ we obtain
\[
\Delta_l^* \leq \left(1 - \frac{1}{\beta_\pi(f,p) \cdot c^*/\gamma^s_{n,k}(f,p)+1}\right)^l \cdot \Delta_0^* \leq \exp\left({-\frac{l}{\beta_\pi(f,p) \cdot c^*/\gamma^s_{n,k}(f,p)+1}}\right) \cdot \Delta_0^*.
\]
Since $f$ is non-negative, we have $f_{\mathrm{avg}}(\pi_0, p) \geq 0$.
Therefore, by the definition of $\Delta_0^*$, $\Delta_0^* \leq f_{\mathrm{avg}}(\pi^*, p)$. Hence,
\[
\Delta_l^* \leq \exp\left({-\frac{l}{\beta_\pi(f,p) \cdot c^*/\gamma^s_{n,k}(f,p)+1}}\right) \cdot f_{\mathrm{avg}}(\pi^*, p).
\]
Therefore, by the definition of $\Delta_l^*$, it follows that 
\[
f_{\mathrm{avg}}(\pi, p) \geq \left(1 - \exp({-\frac{l}{\beta_\pi(f,p)  \cdot c^*/\gamma^s_{n,k}(f,p)+1}})\right) \cdot f_{\mathrm{avg}}(\pi^*, p).
\]
This completes the proof.
\end{proof}
Next, we prove \thmref{esfandiari cavg}.
\begin{proof}[of \thmref{esfandiari cavg}]
Let $c^* = c_{\mathrm{avg}}(\pi^*, p)$.
Define
\[
l := \lceil(\beta_\pi(f,p) \cdot c^*/\gamma^s_{n,k}(f,p) + 1) \cdot \log(nQ/\eta)\rceil.
\]
By \thmref{esfandiari favg}, we have
\begin{align*}
f_{\mathrm{avg}}(\pi_l, p) &\geq \left(1 - \exp({-\frac{l}{\beta_\pi(f,p)  \cdot c^*/\gamma^s_{n,k}(f,p)+1}})\right) \cdot f_{\mathrm{avg}}(\pi^*, p)\\
&\geq \left(1 - \exp({-\frac{(\beta_\pi(f,p) \cdot c^*/\gamma^s_{n,k}(f,p)+1) \cdot \log(nQ/\eta)}{\beta_\pi(f,p) \cdot c^*/\gamma^s_{n,k}(f,p)+1}})\right) \cdot f_{\mathrm{avg}}(\pi^*, p)\\
&= \left(1 - \exp({-\log(nQ/\eta)})\right) \cdot f_{\mathrm{avg}}(\pi^*, p)\\
&= \left(1 - \frac{\eta}{nQ}\right) \cdot f_{\mathrm{avg}}(\pi^*, p)\\
&=Q - \eta/n,
\end{align*}
Where the last inequality follows since $f_{\mathrm{avg}}(\pi^*, p) = Q$ by the assumptions of the theorem.

Denote by $F$ the (random) value of $f$ when $\pi_l$ terminates. By definition, $\E[F] = f_{\mathrm{avg}}(\pi_l, p)$.
By the definition of $f$ and $Q$, $F \leq Q$ with probability one. Therefore, the random variable $X := Q - F$ is non-negative.
By Markov's inequality, $\P[X \geq \eta] \leq \E[X]/\eta.$
Hence,
\begin{align*}
&\P[F > Q - \eta] = 1 - \P[F \leq Q - \eta] = 1 - \P[X \geq \eta] \geq 1 - \E[X]/\eta\\
&\quad= 1 - \frac{Q - f_{\mathrm{avg}}(\pi_l, p)}{\eta} \geq 1 - \frac{Q - (Q - \eta/n)}{\eta} = 1 - \frac{1}{n}.
\end{align*}
By the assumptions of the theorem, if $f(\dom(\psi), \phi)>Q-\eta$, then $f(\dom(\psi), \phi) = Q$ for every $\psi \in \Psi_V$ and $\phi \in \Phi_V$. Therefore, with probability at least $1 - 1/n$, we have $F = Q$.
Let $t := \E[|E(\pi_l, \phi)| \mid F=Q]$ be the expected cost of $\pi_l$ conditioned on $F = Q$.
Then $l = c_{\mathrm{avg}}(\pi_l, p) \geq (1 - 1/n) \cdot t$, thus $t \leq l/(1-1/n)$.
If $F \neq Q$, which happens with probability at most $1/n$, then $\pi$ selects at most $n$ elements. Hence,
\[
c_{\mathrm{avg}}(\pi, p) \leq (1 - \frac{1}{n}) \cdot t + \frac{1}{n} \cdot n \leq l + 1 \leq (\frac{\beta_\pi(f,p) \cdot c^*}{\gamma^s_{n,k}(f,p)} + 1) \cdot \log(\frac{nQ}{\eta}) + 2.
\]
This completes the proof.
\end{proof}

\section{Proofs for \secref{modified}}
First, we prove \thmref{modified}.
\label{app:modified}
\begin{proof}[of \thmref{modified}]
By the definition of $p'$, for any realization $\phi$, we have  $p'(\phi) := \max\{p(\phi), 1/|\Phi_V|^2\} / Z$,
where $Z$ is a normalizing constant.
We have
\begin{equation}\label{eq:Z upper bound}
Z = \sum_{\phi \in \Phi_V} \max\{p(\phi), 1/|\Phi_V|^2\}  \leq 
\sum_{\phi \in \Phi_V} (p(\phi) + 1/|\Phi_V|^2) = 1 + 1/|\Phi_V|.
\end{equation}
Therefore, for any realization $\phi$,
\[
p'(\phi) = \max\{p(\phi), 1/|\Phi_V|^2\} / Z \geq p(\phi)/Z \geq
p(\phi) \cdot \frac{|\Phi_V|}{|\Phi_V|+1}.
\]
Hence,
\begin{align*}
c_{\mathrm{avg}}(\pi, p) &= \sum_{\phi \in \Phi_V} p(\phi) |E(\pi, \phi)|\leq \frac{|\Phi_V|+1}{|\Phi_V|}  \sum_{\phi \in \Phi_V} p'(\phi) |E(\pi, \phi)|= \frac{|\Phi_V|+1}{|\Phi_V|}  c_{\mathrm{avg}}(\pi, p').
\end{align*}
Therefore,
\[
c_\mathrm{avg}(\pi, p) \leq \frac{|\Phi_V|+1}{|\Phi_V|}\cdot c_{\mathrm{avg}}(\pi, p') \leq 2 c_{\mathrm{avg}}(\pi, p').
\]
By \thmref{esfandiari cavg},
\begin{align*}
c_\mathrm{avg}(\pi, p') &\leq \left(\frac{\beta_\pi(f,p') }{\gamma^s_{n,k}(f, p')}\cdot c_\mathrm{avg}(\pi^*, p') + 1\right)  \log(\frac{nQ}{\eta}) + 2.
\end{align*}
Therefore,
\[
c_\mathrm{avg}(\pi, p) \leq 2 \left(\frac{\beta_\pi(f,p') }{\gamma^s_{n,k}(f, p')}\cdot c_\mathrm{avg}(\pi^*, p') + 1\right) \log(\frac{nQ}{\eta}) + 4.
\]

By the assumptions of the theorem, $\pi^*$ has height at most $|\Phi_V|$.
In other words, for any realization $\phi$,
\begin{equation}\label{eq:progressive policy inequality}
|E(\pi^*, \phi)| \leq |\Phi_V|.
\end{equation}
In addition, note that
\[
Z = \sum_{\phi \in \Phi_V} \max\{p(\phi), 1/|\Phi_V|^2\} \geq \sum_{\phi \in \Phi_V} p(\phi) = 1.
\]
Therefore, for any realization $\phi$, if $p(\phi) \geq 1/|\Phi_V|^2$, then
$p'(\phi)-p(\phi) = p(\phi)/Z-p(\phi) \leq 0$. Otherwise, $p(\phi) < 1/|\Phi_V|^2$, which implies $p'(\phi) = 1/(|\Phi_V|^2  Z) \leq 1/|\Phi_V|^2$. We conclude that for any realization $\phi$, $p'(\phi)-p(\phi) \leq 1/|\Phi_V|^2$. Combining this with \eqref{progressive policy inequality}, we get
\begin{align*}
c_{\mathrm{avg}}(\pi^*, p') - c_{\mathrm{avg}} (\pi^*, p) &= \sum_{\phi \in \Phi_V} |E(\pi^*, \phi)|(p'(\phi)-p(\phi))\\
&\leq \sum_{\phi \in \Phi_V} |\Phi_V| (p'(\phi)-p(\phi))\\
&\leq \sum_{\phi \in \Phi_V} |\Phi_V| \frac{1}{|\Phi_V|^2} = 1.
\end{align*}
We conclude that $c_{\mathrm{avg}}(\pi^*, p') \leq c_{\mathrm{avg}} (\pi^*, p) + 1$.
Therefore,
\begin{align*}
c_\mathrm{avg}(\pi, p) \leq 2 \left(\frac{\beta_\pi(f,p') }{\gamma^s_{n,k}(f, p')}\cdot (c_{\mathrm{avg}} (\pi^*, p) + 1) + 1\right) \log(\frac{nQ}{\eta}) + 4.
\end{align*}
This completes the proof.
\end{proof}

Next, we state and prove \lemref{optimal policy always eliminates realization}.
\begin{lemma}\label{lem:optimal policy always eliminates realization}
Let $\pi^*$ be a policy such that $f_p(E(\pi^*, \phi), \phi) = 1$ for all $\phi \in \Phi_V$ and such that $c_{\mathrm{avg}}(\pi^*,p)$ is minimal.
Then, the height of $\pi^*$ is at most $|\Phi_V|$.
\end{lemma}

\begin{proof}
For simplicity, we denote below $\pi = \pi^*$. We prove the lemma by showing that $\pi$ eliminates at least one realization from its version space in each query. This directly implies that the height of $\pi$ is at most $|\Phi_V|-1$.

We first assume that $\pi$ is deterministic, and later extend the analysis to randomized policies.
Assume for contradiction that there is a possible selection by $\pi$ such that the resulting observation does not remove any realization from the version space.  Formally, assume that there is some partial realization $\psi$, realization $\phi \sim \psi$ with some non-zero probability according to $p$, and an element $v \in V$, such that during the run of $\pi$ under $\phi$, the partial realization that is observed is $\psi$, $\pi(\psi) = v$, and for any realization $\phi' \sim \psi$ it holds that $\phi'(v) = \phi(v)$.

Let $\pi'$ the policy defined as follows: For any partial realization $\psi'$: 
\[
\pi'(\psi') = \begin{cases}
\pi(\psi' \cup \{(v, \phi(v))\}) & \psi \subseteq \psi',\\
\pi(\psi') & \text{otherwise}.
\end{cases}
\]
By the definition of $\pi'$, it holds that $c_{\mathrm{avg}}(\pi', p) < c_{\mathrm{avg}}(\pi, p)$.
Thus, to reach a contradiction, it suffices to show that $f_p(E(\pi', \phi'), \phi') = 1$ for all $\phi' \in \Phi_V$.

Let $\phi' \in \Phi_V$.
If $\psi$ is not observed during the run of $\pi'$ under $\phi'$, then the behavior of $\pi'$ is equivalent to that of $\pi$. Therefore,
\[
f_p(E(\pi', \phi'), \phi') = f_p(E(\pi, \phi'), \phi') = 1.
\]

Otherwise, $\psi$ is observed during the run of $\pi'$ under $\phi'$.
In this case, until $\psi$ is observed, the behavior of $\pi'$ under $\phi'$ is the same as its behavior under $\phi$. Therefore, at some time during the run of $\pi'$, $\psi$ is the partial realization that has been observed so far. By the definition of $\pi'$, it does not select $v$ at this time, and then behaves the same as $\pi$. Hence, $E(\pi', \phi') = E(\pi, \phi') \setminus \{v\}$.
For any deterministic policy $\Bar{\pi}$, we define $\psi_{\Bar{\pi}} = \{(v, \phi'(v)) \mid v \in E(\Bar{\pi}, \phi')\}$ to be the partial realization that $\Bar{\pi}$ observes under $\phi'$ just before it terminates. It holds that $\psi_{\pi} = \psi_{\pi'} \cup \{(v, \phi'(v))\}$.

Since $\psi$ is observed during the run of $\pi'$ under $\phi'$, it holds that $\psi \subseteq \psi_{\pi'}$. Hence, any realization that is consistent with $\psi_{\pi'}$ is also consistent with $\psi$.
We assumed that any realization that is consistent with $\psi$ is also consistent with $\{(v, \phi'(v))\}$.
Therefore, the same holds for $\psi_{\pi'}$.
We conclude that any realization that is consistent with $\psi_{\pi'}$ is also consistent with $\psi_{\pi}$.
Since $f_p(E(\pi, \phi'), \phi') = 1$, it holds that $\phi'$ is the only realization that is consistent with $\psi_{\pi}$. Hence, $\phi'$ is also the only realization that is consistent with $\psi_{\pi'}$.
By the definition of $f_p$, it holds that $f_p(E(\pi', \phi'), \phi') = 1$.
Thus, in all cases, $f_p(E(\pi', \phi'), \phi') = 1$ for all $\phi' \in \Phi_V$, as claimed. This shows the required contradiction.

Lastly, if $\pi$ is not deterministic, then for any realization $\theta_{\pi}$ of the randomness of $\pi$, it holds that $\pi$ given $\theta_{\pi}$ is deterministic.
Assume for contradiction that there exists some realization $\theta_{\pi}$ of the randomness of $\pi$ that has a non-zero probability, such that in some query, $\pi$ under $\theta_{\pi}$ does not eliminate any realization from its version space.
Since $\pi$ under $\theta_{\pi}$ is a deterministic policy, similarly to the proof above, it follows that $\pi$ under $\theta_{\pi}$ is not an optimal policy.
Since $\theta_{\pi}$ has a positive probability, it follows that $\pi$ is not an optimal policy, again resulting in a contradiction and proving the claim.
\end{proof}
\end{document}